\documentclass[10pt,journal,compsoc]{IEEEtran}
\ifCLASSOPTIONcompsoc
  \usepackage[nocompress]{cite}
\else
  \usepackage{cite}
\fi
\ifCLASSINFOpdf
\else
\fi
\usepackage{graphicx}
\usepackage{amsmath,amssymb} % define this before the line numbering.
\usepackage{color}

\usepackage{amsthm}
\usepackage{paralist}
\usepackage{caption}
\usepackage{booktabs}
\usepackage{cite}
\usepackage{multirow}

\newtheorem*{prop}{Proposition}
\newtheorem*{define}{Problem Definition}

\begin{document}
\title{Time Perception Machine: Temporal Point Processes for the When, Where and What of Activity Prediction}

\author{Yatao~Zhong,
        Bicheng~Xu,
        Guang-Tong Zhou,
        Luke~Bornn,
        Greg~Mori
        % <-this % stops a space
\IEEEcompsocitemizethanks{\IEEEcompsocthanksitem Y. Zhong, B. Xu and G. Zhou were with School of Computing Science, Simon Fraser University, Burnaby, BC, Canada.\protect\\
E-mail: yataozhong@gmail.com, bichengxu@gmail.com, zhouguangtong@gmail.com
\IEEEcompsocthanksitem L. Bornn and G. Mori are with School of Computing Science, Simon Fraser University, Burnaby, BC, Canada.\protect\\
E-mail: lbornn@sfu.ca, mori@cs.sfu.ca}% <-this % stops an unwanted space
}

\IEEEtitleabstractindextext{%
\begin{abstract}
Numerous powerful point process models have been developed to understand temporal patterns in sequential data from fields such as health-care, electronic commerce, social networks, and natural disaster forecasting. In this paper, we develop novel models for learning the temporal distribution of human activities in streaming data (e.g., videos and person trajectories). We propose an integrated framework of neural networks and temporal point processes for predicting when the next activity will happen.  Because point processes are limited to taking event frames as input, we propose a simple yet effective mechanism to extract features at frames of interest while also preserving the rich information in the remaining frames. We evaluate our model on two challenging  datasets. The results show that our model outperforms traditional statistical point process approaches significantly, demonstrating its effectiveness in capturing the underlying temporal dynamics as well as the correlation within sequential activities. Furthermore, we also extend our model to a joint estimation framework for predicting the timing, spatial location, and category of the activity simultaneously, to answer the \textit{when}, \textit{where}, and \textit{what} of activity prediction.
\end{abstract}

\begin{IEEEkeywords}
Temporal Point Process, Activity Prediction, Time Perception Machine
\end{IEEEkeywords}}

\maketitle

\IEEEdisplaynontitleabstractindextext

% \fi
\IEEEpeerreviewmaketitle

\IEEEraisesectionheading{\section{Introduction}\label{sec:introduction}}
During the past decades, researchers have made substantial progress in computer vision algorithms that can automatically detect~\cite{shou2016action, shou2017cdc, yuan2017temporal} and recognize ~\cite{karpathy2014large, yue2015beyond, feichtenhofer2016convolutional, qiu2017learning} actions in video sequences. However, the ability to go beyond this and estimate how past actions will affect future activities opens exciting possibilities. A good estimation of future behaviour is an essential sensory component for an automated system to fully comprehend the real world. In this paper, we tackle the problem of estimating the prospective occurrence of future activity. Our goal is to predict the timing, spatial location, and category of the next activity given past information.  We aim to answer the \textit{when}, \textit{where}, and \textit{what} questions of activity prediction. 

Consider the sports video example shown in Fig.~\ref{fig:pull}.  In our work, we directly model the occurrence of discrete activity events that occur in a data stream.  Within a sports context, these activities could include key moments in a game, such as passes, shots, or goals.  More generally, they could correspond to important human actions along a sequence: such as a person leaving a building, stopping to engage in conversation with a friend, or sitting down on a park bench.  Predicting where and when these semantically meaningful events occur would enable many applications within robotics, autonomous vehicles, security and surveillance, and other video processing domains.

\begin{figure*}[t]
  \centering
  \includegraphics[width=0.95\textwidth]{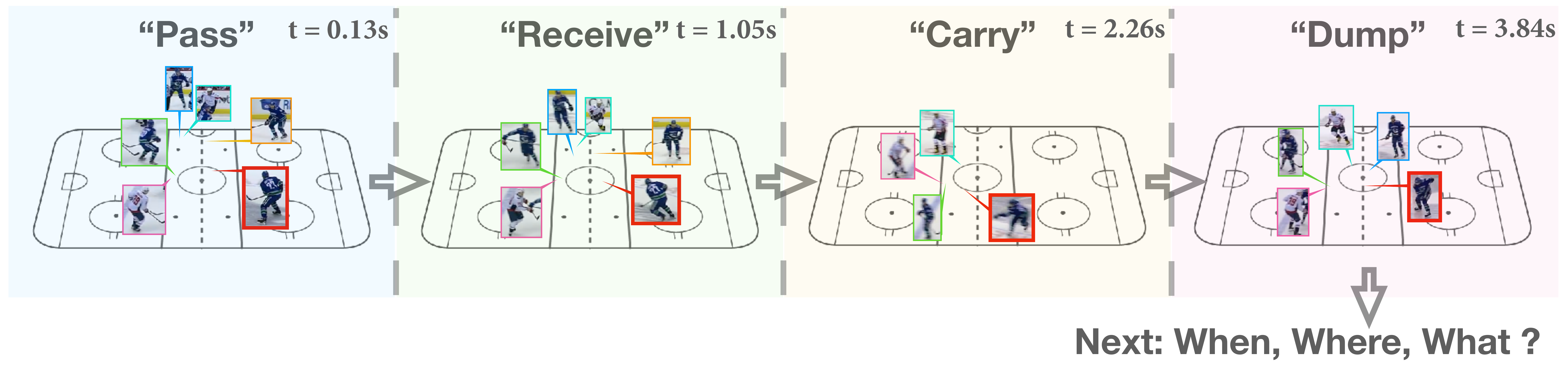}
  \caption{An ice hockey example: 1) the puck is passed to the player in the red box; 2) the player in the red box receives the puck; 3) the player in the red box carries the puck across the centre line; 4) the player in the red box dumps the puck into the offensive zone. Given the sequence of activities above, we aim to predict what the next activity will be, where it will take place, and when it will occur.}
  \label{fig:pull}
\end{figure*}

\begin{define} 
Let the input be a sequence of $n$ frames. Among these, $j$ ($j \ll n$) frames are each marked by an activity, whose timestamps are denoted as $\{t_1, t_2, \cdots, t_j \}$.  Our goal is to estimate when and where the next activity ($j+1$) will happen and what type of activity it will be given the past sequence of activities and frames up to $t_j$.
\end{define}

Importantly, we are interested in predictions regarding the semantically meaningful, sparsely occurring events within a sequence.    This discrete time moment representation for actions is commonplace in numerous applications: e.g., where and when will the next shot take place in this hockey game, where do we need to be to intercept it; from where and when will the next person hail a rideshare, where should we drive to pick him/her up; when is the next nursing home patient going to request assistance, what will he/she request and where will that request be made?  Generalizations of this paradigm are possible, where we consider multiple people, such as players in a sports game.  We elaborate on this idea and demonstrate that we can model events corresponding to important, actionable inferences.

Following the standard terminology \cite{lian2015multitask}, we use the term \textit{\textbf{arrival pattern}} to refer to the temporal distribution of activities throughout the paper.  We wish to model this distribution and infer when and where the next activity will take place.  However, in vision tasks the raw input has $n$ frames, whereas we are interested in the $j$ moments sparsely distributed in the sequence that are the points at which activities commence.  Therefore, we need a mechanism to build features from the $j$ frames while also preserving information of other regular frames. To address this problem, we utilize a hierarchical recurrent neural network with skip connections for multi-resolution temporal data processing.

Similar to variational autoencoders \cite{kingma2013auto,gregor2015draw}, which model the distribution of latent variables with deep learning, our model leverages the same advantage of neural networks to fit the arrival pattern (temporal distribution of activities) in the data. A network is used to learn the conditional intensity of a temporal point process and the likelihood is maximized during training. In contrast to traditional statistical approaches that demand expert domain knowledge, our model does not require a hand-crafted conditional intensity. Instead, it is automatically learned on top of raw data. We name our model the Time Perception Machine (TPM).

Our work has three main contributions: 
\begin{enumerate}[(i)]
\item Proposing a new task -- predicting the occurrence of future activity -- for human action analysis, which has not been explored before on streaming data such as videos and person trajectories; 
\item Developing a novel hierarchical RNN with skip connections for feature extraction at finer resolution (frames of interest) while preserving information at coarser resolution; 
\item Formulating a generic conditional intensity and extending the model to a joint prediction framework for the when, where and what of activity forecasting. 
\end{enumerate}

\section{Related Work} 
\label{sec:relatework}

\subsection{Activity Forecasting}

Seminal work on activity forecasting was done by Kitani et al.~\cite{kitani2012activity}, who modeled the effect of physical surroundings using semantic scene labeling and inverse reinforcement learning to predict plausible future paths and destinations of pedestrians.   

Subsequent work~\cite{rhinehartfirst17} reasons about the long-term behaviors and goals of an individual given his first-person visual observations. Similarly, Xie et al.~\cite{Xie2016ModelingAI} attempted to infer human intents by leveraging the agent-based Lagaragian mechanics to model the latent needs that drive people toward functional objects.
Park et al.~\cite{soo2016egocentric} proposed an EgoRetinal map for motion planning from egocentric stereo videos.
Vondrick et al.~\cite{Vondrick2016AnticipatingTF} presented a framework for predicting the visual representations of future frames, which is employed to anticipate actions and objects in the future.  Unlike the previous work on activity forecasting, which focuses on planning paths and predicting intent, our work addresses a different problem in that we aim to predict the discrete attributes (the when, where, and what) of future activities.

Recent temporal activity detection / prediction methods build on recurrent neural network architectures.  These include connectionist temporal classification (CTC) architectures~\cite{huang2016connectionist,graves2006connectionist}.  CTC models conduct {\it classification} by generalizing away from actual time stamps, while prediction methods regress actual temporal values. A variety of temporal neural network structures exist (convolutional~\cite{2018arXiv180301271B}, GRU, LSTM, Phased LSTM~\cite{neil2016phased}), many of which have been applied to activity recognition.  Our contribution is complementary in that it focuses on a novel point process model for distributions of discrete events for activity prediction.

\subsection{Temporal Point Processes}
A temporal point process is a stochastic model used to capture the arrival pattern of a series of events in time. Temporal point processes are studied in various areas including health-care analysis \cite{lasko2014efficient}, electronic commerce \cite{xu2014path}, modeling earthquakes and aftershocks \cite{ogata1998space}, etc. 

%\footnote{We use ``$^*$'' to indicate that a quantity is conditioned on the past information throughout this paper.}.
A temporal point process model can be fully characterized by the ``conditional intensity'' quantity, denoted by $\lambda(t|H)$, which is conditioned on the past information $H$. The conditional intensity encodes the expected rate of arrivals within an infinitesimal neighborhood at time $t$. Once we determine the intensity, we determine a temporal point process. Mathematically, given the history $H_j$ up to the event $j$ and the conditional intensity $\lambda(t|H_j)$, we can formulate the probability density function $f(t|H_j)$ and the cumulative distribution function $F(t|H_j)$ for the time of the next event $(j+1)$, shown in Eq.~\ref{eq:density} and Eq.~\ref{eq:prob}. We defer the full derivation of both formulas to Appendix \ref{appdx:pdf_cdf}.
\begin{eqnarray}
    f(t|H_j) &=& \lambda(t|H_j) \cdot e^{-\int_{t_j}^{t} \lambda(u|H_j) du} \label{eq:density} \\
    F(t|H_j) &=& 1-e^{-\int_{t_j}^{t} \lambda(u|H_j) du} \label{eq:prob}
\end{eqnarray}

For notational convenience, we use ``$^*$'' to indicate that a quantity is conditioned on the past throughout this paper. For example, $\lambda^*(t) \triangleq \lambda(t|H_j)$, $f^*(t) \triangleq f(t|H_j)$ and $F^*(t) \triangleq F(t|H_j)$. Below we show the conditional intensities of several temporal point process models.

% \gm{I find these hard to read without knowing how the probability density function is going to be formulated.  Is it possible to either (1) reverse the order here so that f(t) and F(t) are defined first, or (2) provide some hand-waving description of how the conditional intensity works (more than what is above.} \yt{Fixed. I reversed the order.}

\textbf{Poisson Process} \cite{kingman1993poisson}. $\lambda^*(t)=\lambda$, where $\lambda$ is a positive constant.

\textbf{Hawkes Process} \cite{hawkes1971spectra}. $\lambda^*(t)=\lambda+\alpha \sum_{t_i<t} e^{-\gamma(t-t_i)}$, where $\lambda$, $\alpha$ and $\gamma$ are positive constants. This process is an ``aggregated'' process, where one event is likely to trigger a series of other events in a short period of time, but the likelihood drops exponentially with regard to time.  

\textbf{Self-Correcting Process} \cite{isham1979self}. $\lambda^*(t)=e^{\mu t-\sum_{t_i<t} \alpha}$, where $\mu$ and $\alpha$ are positive constants. This process is more ``averaged'' in time. A previous event is likely to inhibit the occurrence of the next one (by decreasing the intensity). Then the intensity will increase again until the next event happens. 

Furthermore, a recent work by Du et al.~\cite{du2016recurrent} explored temporal process models using neural networks, but only experimented with sparse timestamp data. We extend their approach to dense streaming data with the proposed hierarchical RNN to extract features at frames of interest. Additionally, we demonstrate the effectiveness of a more generic intensity function in modeling the arrival pattern. We also show how a more powerful joint estimation framework can be formulated for simultaneous prediction of the timing, spatial location and category of the next activity event.

\section{Model}
We will first introduce the hierarchical RNN structure upon which our model is built. Then we will present in detail the formulation and derivation of the proposed model for predicting the timing of future activities. Finally we show how our model can be extended to a joint estimation framework for the simultaneous prediction of the time, location, and category of the next activity. 

\subsection{Hierarchical RNN}
The input to our model is an entire sequence of $n$ frames. In our experiments, these include visual data in the form of bounding boxes cropped around people in video sequences and/or representations of human motion trajectories as 2D coordinates of person location over time. 

A typical temporal point process model only takes as input the $j$ frames annotated with activities. These are very sparse compared to the entire dense sequence of $n$ frames ($j \ll n$). We expect these $j$ significant frames will contain important features. However, we do not want to lose any information inherent in the remaining ($n-j$) frames. To this end, we need a hierarchical RNN capable of feature extraction at different time resolutions. This is similar in vein to tasks from the natural language processing domain, such as recent work \cite{chung2016hierarchical, sordoni2015hierarchical, kong2015segmental} in language modeling, with character-to-word and word-to-phase networks for feature extraction at multiple scales.  More generally, this is an instance of the classic
multiple-time scales problem in recurrent neural networks~\cite{el1996hierarchical}.

In our case, we use a hierarchical RNN model composed of two stacked RNNs. The lower-level RNN looks into the details by covering every frame in the input sequence. The higher level RNN fixes its attention only on frames of activities so as to capture the temporal dynamics among these significant times. We implement the RNN with LSTM cells. Fig.~\ref{fig:hrnn} shows the model structure.

\begin{figure*}[t]
\centering
\includegraphics[width=0.92\linewidth]{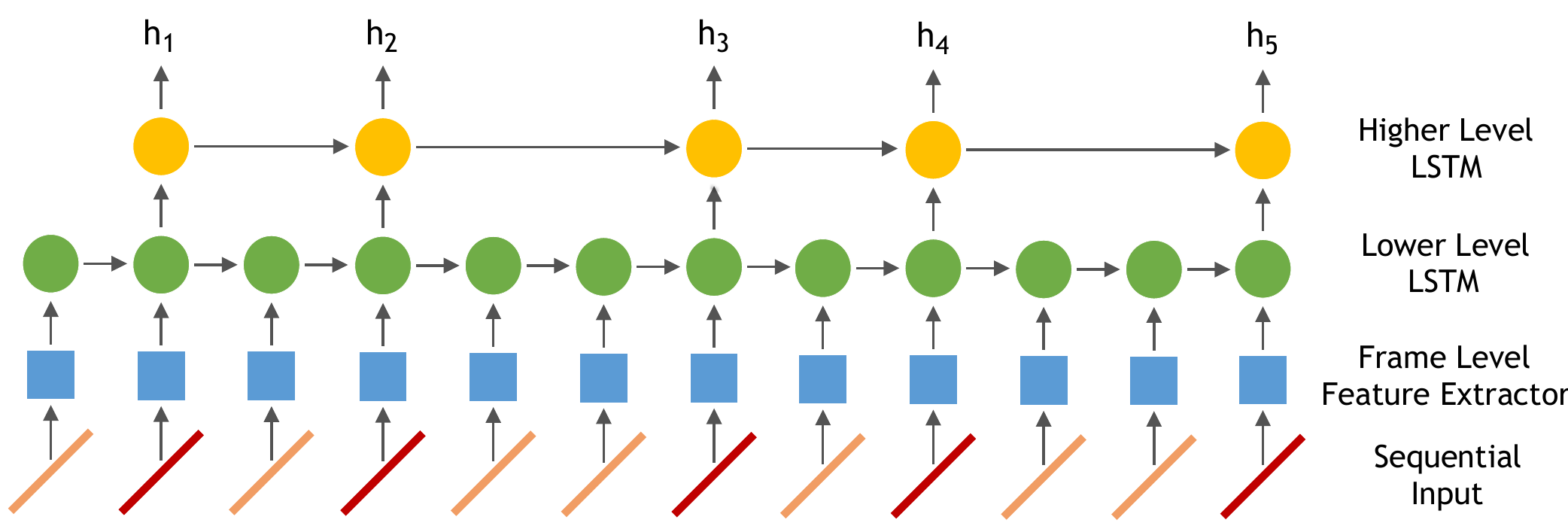} 
\caption{The hierarchical RNN structure. The frame level feature extractor can be any network applied to frames (e.g., VGG-16 net \cite{simonyan2014very}). The dense sequence of $n$ frames is fed into the lower level LSTM while only the significant $j$ frames pass their features to the higher level LSTM for further processing.}
\label{fig:hrnn}
\end{figure*}

\subsection{Conditional Intensity Function}
Instead of hand-crafting the conditional intensity $\lambda^*(t)$, we view it as the output of the hierarchical RNN and learn the conditional intensity directly from raw data. However, an arbitrary choice of the conditional intensity $\lambda^*(t)$ could be potentially problematic, because it needs to characterize a probability distribution. Thus, we need to validate the resultant probability density function in Eq.~\ref{eq:density} and the cumulative distribution function in Eq.~\ref{eq:prob}.
\begin{prop}
\begin{inparaenum}[]
$\lambda^*(t)$ is a valid conditional intensity that defines a temporal point process if and only if it satisfies \item $\int_{t_j}^{\infty} \lambda^*(t) dt=\infty$ and $\lambda^*(t) >0, \forall t$. 

\end{inparaenum}
\end{prop}

\begin{proof}
%By definition, $\lambda^*(t)$ is positive. Now we prove condition (\ref{prop:cond2}). 

Necessity ($\Leftarrow$).  Given $\int_{t_j}^{\infty} \lambda^*(t) dt=\infty$ and Eq.\ \ref{eq:prob}, we have $F^*(\infty)=1$, from which it follows that $P(t_j\leq t\leq \infty)=1$. Since $\lambda^*(t)$ is positive, under this condition it defines a valid probability distribution, hence a well established temporal point process.  

Sufficiency ($\Rightarrow$). First, $\lambda^*(t)$ must be positive for it to define a valid probability density by Eq.~\ref{eq:density}. If $\int_{t_j}^{\infty} \lambda^*(t) dt \neq \infty$, which means the integral is a positive less than $\infty$, then it is easy to notice that $F^*(\infty)<1$. This would be an invalid cumulative distribution function since $P(t_j\leq t\leq \infty)<1$.
\end{proof}

\noindent We formally define two forms of conditional intensity as follows.

\textbf{Explicit time dependence $\lambda_A^*$:} The first form is inspired by \cite{du2016recurrent}, which models the conditional intensity based on the hidden states $h_j$ and the time $t$.
\begin{equation}
\lambda_A^*(t)=e^{vh_j+w(t-t_j)+b}, \ \ s.t.\ \ w>0
\label{eq:exp_intensity}
\end{equation}
Note that we make an important correction to \cite{du2016recurrent}. The conditional intensity without the positive constraint in Eq.\ \ref{eq:exp_intensity} does not conform to the necessary condition above. By imposing a constraint $w>0$, we can prove that the revised intensity in Eq.\ \ref{eq:exp_intensity} satisfies the condition in the above proposition.

%Our model utilizes the conditional intensities formally defined in Eq.\ \ref{eq:exp_intensity} and Eq.\ \ref{eq:simple_intensity} below. We obtain Eq.\ \ref{eq:exp_intensity} by making an important correction to the intensity used in \cite{du2016recurrent}. We found that the conditional intensity without the positive constraint in Eq.\ \ref{eq:exp_intensity} does not conform to the condition (\ref{prop:cond2}) above. By imposing a constraint $w \geq 0$, we can prove that the revised one in Eq.\ \ref{eq:exp_intensity} satisfies the condition (\ref{prop:cond2}).
% \begin{equation} \label{eq:exp_intensity}
%     \lambda^*(t)=e^{vh_j+w(t-t_j)+b}, \ \ s.t.\ \ w\geq0
% \end{equation}

\textbf{Implicit time dependence $\lambda_B^*$:}
Note that the design of $\lambda_A^*(t)$, to some extent, assumes how it is a function of time $t$. As $t$ is part of the input, we believe it is possible to acquire the time information from the hidden states $h_j$ without any specification about $t$. We use an exponential activation to ensure the positivity of the resultant conditional intensity. Formally, we have:
\begin{equation}
\lambda_B^*(t)=e^{wh_j+b}
\label{eq:simple_intensity}
\end{equation}

%\subsubsection{Validity: } 
The proof for the validity of $\lambda_A^*(t)$ and $\lambda_B^*(t)$ is provided in Appendix \ref{appdx:valid_intensity}. The analytic form for the likelihood $f^*(t)$ is obtained by substituting Eq. \ref{eq:exp_intensity} or Eq. \ref{eq:simple_intensity} into Eq. \ref{eq:density}:
\begin{equation}
\label{eq:exp_pdf}
\begin{aligned}
    f^*_A(t) =& exp\Big\{ vh_j+w(t-t_j)+b \\
        &- \frac{1}{w} e^{vh_j+b} \big[e^{w(t-t_j)}-1 \big] \Big\}, \\ 
        &s.t.\ \ w>0;
\end{aligned}
\end{equation}
\begin{equation}
\label{eq:simple_pdf}
\begin{aligned}
    f^*_B(t) =& exp\Big\{ wh_j+b - e^{wh_j+b}(t-t_j) \Big\}.
\end{aligned}
\end{equation}

% \subsubsection{Visualizations:}
% Fig. \ref{fig:density_plots} visualizes both likelihood functions, the shapes of which are controlled by the hidden states and trainable parameters in Eq.\ \ref{eq:exp_intensity} and Eq.\ \ref{eq:simple_intensity}. If the model is only expected to predict the timing of the next activity, we only need to maximize this likelihood to fit the arrival pattern in the data; otherwise we show how to do joint estimation in the next section.

\subsection{Joint Likelihood}
Now we show our model can be readily plugged into a joint estimation framework by formulating a joint likelihood for the timing, spatial location and category of activities. However, instead of directly modeling the next activity location, we use an incremental approach that models the space shift from the current position. Let $L$ be the joint likelihood for a sequence of activities; $t$, $a$ and $s$ denote the timestamp, action category, and space shift respectively. To derive the joint likelihood, we make the following assumption.

% \begin{assumption}
%     Action category and space shift as a whole at different timestamps are independent, but they both depend on the history: %$f^*(a_i, s_i, t_i, a_j, s_j, t_j)=f^*(a_i, s_i, t_i) f^*(a_j, s_j, t_j)$, given $t_i<t_j$.
%     \begin{equation}
%     \begin{aligned}
%         f(a_i, s_i, t_i, a_j, s_j, t_j | H_j) &= f(a_i, s_i, t_i | H_i) f(a_j, s_j, t_j | H_j), \\ 
%         \text{where the history  } H_i &=\{(a_k, s_k, t_k) | k < i\} \text{  and  } i<j.
%     \end{aligned}
%     \end{equation}
% \end{assumption}
% \begin{assumption} \label{assumpt:action_space_indep}
%     Action category and space shift of event $i$ are conditionally independent given the history up to event ($i-1$):  $H_i=\{(t_k, a_k, s_k)|k<i\}$: $f(a_i, s_i|H_i)=f(a_i|H_i) f(s_i|H_i)$. 
% \end{assumption}

For mathematical convenience, we assume the timing, action category, space shift of event $j$ are conditionally independent given the history $H_j=\{(t_k, a_k, s_k)|k<j\}$ up to event ($j-1$). That is, $f(t_j, a_j, s_j|H_j)=f(t_j|H_j) f(a_i|H_j) f(s_j|H_j)$, or $f^*(t_j, a_j, s_j)=f^*(t_j)f^*(a_j)f^*(s_j)$ if we use the ``*" notation. Therefore, we have the joint likelihood $L_{\theta}$ parameterized by $\theta$:
\begin{equation}
\label{eq:likelihood}
\begin{aligned}
    L_{\theta}=& f_{\theta}(t_1, a_1, s_1, t_2, a_2, s_2, \cdots, t_n, a_n, s_n) \\
    =& f_{\theta}(t_1, a_1, s_1) f_{\theta}(t_2, a_2, s_2|t_1, a_1, s_1) \\ & f_{\theta}(t_3, a_3, s_3|t_1, a_1, s_1, t_2, a_2, s_2) \cdots \\ & f_{\theta}(t_n, a_n, s_n|t_1, a_1, s_1, t_2, a_2, s_2, \cdots, \\
    & t_{n-1}, a_{n-1}, s_{n-1}) \\
    =& f_{\theta}^*(t_1, a_1, s_1) f_{\theta}^*(t_2, a_2, s_2) \cdots f_{\theta}^*(t_n, a_n, s_n) \\
    =& f_{\theta}^*(t_1) f_{\theta}^*(a_1) f_{\theta}^*(s_1) f_{\theta}^*(t_2) f_{\theta}^*(a_2) f_{\theta}^*(s_2) \\
    & \cdots f_{\theta}^*(t_n) f_{\theta}^*(a_n) f_{\theta}^*(s_n)
\end{aligned}
\end{equation}
 We drop the subscript ``$\theta$" whenever possible for clean notations. Since we have already obtained the form of $f^*(t)$ in Eq. \ref{eq:exp_pdf} and Eq. \ref{eq:simple_pdf}, in the next section we derive the form of $f^*(a)$ and $f^*(s)$.

\textbf{Estimating the Action Category:}
The action category likelihood $f^*(a)=f(a|H)$ represents the distribution over the type of action. Since the history $H$ is encoded by the RNN hidden states $h$, we have $f^*(a)=f(a|h)$. Given the hidden states $h$, our model outputs a discrete distribution over $K$ action classes:
\begin{equation}
    \hat{P}_a=softmax(w'h+b')=[\hat{p}_{a,1}, \hat{p}_{a,2}, \cdots, \hat{p}_{a,K} ]^T.
\end{equation}
We then model this likelihood with a Gibbs distribution:
\begin{equation} \label{eq:action_likelihood}
    f^*(a)=e^{-D_{KL}(\hat{P_a}||P_a)}
\end{equation}
where the energy function $D_{KL}(\hat{P_a}||P_a)$ is the Kullback-Leibler divergence between the predicted distribution $\hat{P_a}$ and the ground-truth distribution $P_a$ (encoded as a one-hot vector).

\textbf{Estimating the Space Shift:}
The space shift likelihood gives the spatial distribution of the next move. Similar to $f^*(a)$, we have $f^*(s)=f(s|h)$. We model the likelihood using a bivariate Gaussian distribution:
\begin{equation}
\label{eq:spatial_likelihood}
    f^*(s)=\frac{1}{2\pi \sqrt{|\Sigma|}} e^{-\frac{1}{2} (s-\mu)^T \Sigma^{-1} (s-\mu)}
\end{equation}
where $\mu=(\mu_x, \mu_y)$ is the mean and $\Sigma$ is a 2x2 covariance matrix. We find that learning all the parameters in $\Sigma$ is unstable, so we assume the shifts along the $x$ and $y$ directions are independent, hence $\Sigma=\left( \begin{smallmatrix} \sigma_x^2&0\\ 0&\sigma_y^2 \end{smallmatrix} \right)$. We set $\Sigma$ to be constant and given the hidden states $h$; we use
\begin{equation}
    \mu=w''h+b''
\end{equation}
to parameterize Eq.\ \ref{eq:spatial_likelihood}, where $w''$ and $b''$ are learnable parameters.

\subsection{Training}
The model parameters can be learned in a supervised learning framework, by maximizing the likelihood of event sequences.  In order to formulate the data (log-)likelihood, we
substitute \ref{eq:exp_pdf}, \ref{eq:simple_pdf}, \ref{eq:action_likelihood} and \ref{eq:spatial_likelihood} into Eq.\ \ref{eq:likelihood}.  Converting this to log-likelihood yields Eq.\ \ref{eq:exp_log_likelihood} and Eq.\ \ref{eq:simple_log_likelihood} for the intensities $\lambda_A^*(t)$ and $\lambda_B^*(t)$ in Eq.\ \ref{eq:exp_intensity} and Eq.\ \ref{eq:simple_intensity}, respectively.
\begin{equation}
\label{eq:exp_log_likelihood}
\begin{aligned}
    \log L_A=& \sum_{j} \Big\{ -\sum_{k} \hat{p}_{a_j,k} \log \frac{\hat{p}_{a_j,k}}{p_{a_j,k}} \\
    & -\frac{1}{2} \big[\frac{(s_{j,x}-\mu_{j,x})^2}{\sigma_x^2}+\frac{(s_{j,y}-\mu_{j,y})^2}{\sigma_y^2} \big] \\
    & +vh_{j} +b +w(t_{j+1}-t_j) \\
    & -\frac{1}{w} e^{vh_j+b} \big[e^{w(t_{j+1}-t_j)}-1 \big] \Big\} +C, \\
    & s.t.\ \ w> 0
\end{aligned}
\end{equation}

\begin{equation}
\label{eq:simple_log_likelihood}
\begin{aligned}
    \log L_B=& \sum_{j} \Big\{ -\sum_{k} \hat{p}_{a_j,k} \log \frac{\hat{p}_{a_j,k}}{p_{a_j,k}} \\
    & -\frac{1}{2} \big[\frac{(s_{j,x}-\mu_{j,x})^2}{\sigma_x^2} +\frac{(s_{j,y}-\mu_{j,y})^2}{\sigma_y^2} \big] \\
    & +wh_j +b \\
    & - e^{wh_j+b}(t_{j+1}-t_j) \Big\} +C
\end{aligned}
\end{equation}
Here $C$ absorbs all constants in the derivation above and can be dropped during optimization. The joint likelihood for all sample sequences is obtained by summing the log-likelihood for each sequence. Because the log-likelihood is fully differentiable, we can apply back-propagation algorithms for maximization.

\subsection{Inference}
To infer the timing of the next activity, we follow the same inference procedure as in the standard point process literature: given all ground-truth history up to activity $j$, we predict when the next activity $(j+1)$ will happen. Then we proceed to predict the timing of activity $(j+2)$ given all ground-truth history up to activity $(j+1)$. Therefore, the errors will not accumulate exponentially. This is a reasonable approach in many practical scenarios (knowing what has happened up to now, predict the next event). While we have a full model of the distribution, to obtain a point estimate, we take the expected time $\hat{t}_{j+1}=E(t_{j + 1})$ as our prediction. Eq.\ \ref{eq:exp_pred_time} is the result obtained using the conditional intensity $\lambda_A^*(t)$ in Eq.\ \ref{eq:exp_intensity}, where $\Gamma(0, x)$ is an incomplete gamma function whose value can be evaluated using numerical integration algorithms. Eq.\ \ref{eq:simple_pred_time} is acquired using the conditional intensity $\lambda_B^*(t)$ in Eq.\ \ref{eq:simple_intensity}. The derivation makes use of Eq.\ \ref{eq:density}, and we include the full details in the supplementary material.

\begin{equation}
\label{eq:exp_pred_time}
\begin{aligned}
    E_A(t_{j+1}) =& \int_{t_j}^{\infty} t f_A^*(t) dt \\
     =& \frac{\Gamma(0, \eta) \cdot e^{\eta}}{w}+t_j ,\\
      \text{where}\ \eta=& \frac{1}{w} e^{vh_j+b}
\end{aligned}
\end{equation}

\begin{equation}
\label{eq:simple_pred_time}
\begin{aligned}
    E_B(t_{j+1})=& \int_{t_j}^{\infty} t f_B^*(t) dt \\
     =& e^{-(wh_j+b)}+t_j
\end{aligned}
\end{equation}

To predict the category of the next activity, we take the most confident class in the output distribution $\hat{P_a}$ as the prediction:
\begin{equation}
    \hat{a}_{j+1}=\arg \max_{k} \hat{p}_{a_{j+1},k}.
\end{equation}

To estimate the spatial location of the next activity, we take the expected space shift added to the current position $(x_j, y_j)$ as the result:
\begin{eqnarray}
\label{eq:pred_location}
\begin{aligned}
    \hat{x}_{j+1}=& x_j+E(s_{j+1,x}) \\
     =& x_j+\int_{-\infty}^{\infty} s_x \big[\int_{-\infty}^{\infty} f(s_x, s_y|t_{j+1}) ds_y \big] ds_x \\
    =& x_j+\mu_{j+1,x},\\
    \hat{y}_{j+1} =& y_j+E(s_{j+1,y}) \\
    =& y_j+\int_{-\infty}^{\infty} s_y \big[\int_{-\infty}^{\infty} f(s_x, s_y|t_{j+1}) ds_x \big] ds_y \\
    =& y_j+\mu_{j+1,y}.
\end{aligned}
\end{eqnarray}

% \gt{Once we made prediction on $t+1$, how to proceed to $t+2$? Do we assume the prediction on $t+1$ as ``ground-truth'', and then predict $t+2$ based on 1, 2, ..., $t+1$? I think we should clarify this point in text.} \yt{We just follow the convection of using RNN for prediction. Added a sentence in the beginning of section 3.5.}

\section{Experiments}
We evaluate the model on two challenging datasets collected from real world sports games. These datasets include activities in basketball and ice hockey with extremely fast movement. 

All of our baselines consist of two components: a Markov chain and a conventional point process. The Markov chain models action category and space shift distribution; the point process models action timestamps. In our experiments, we compare TPM's performance in time estimation with three other typical temporal point processes: Poisson process, Hawkes process and self-correcting process (Sec.\ \ref{sec:relatework}). We compare TPM's performance in space and category prediction with $k$-order Markov chains ($k=1, 3, 5, 7, 9$). Also note that TPM has two variants, TPM$_A$ and TPM$_B$, using the two conditional intensity functions $\lambda_A^*(t)$ and $\lambda_B^*(t)$ in Eq.\ \ref{eq:exp_intensity} and Eq.\ \ref{eq:simple_intensity}, respectively.

\subsection{Datasets}
\textbf{STATS SportVU NBA dataset.} This dataset contains the trajectories of 10 players and the ball in court coordinates. During each basketball game possession, there are annotations about when and where a pre-defined activity is performed, such as pass, rebound, shot, etc. 

The frame data are obtained by concatenating the $(x, y)$ court coordinates of the offensive players, defensive players and the ball. The order of concatenation within each team is determined by how far a player is away from the ball. The closest is the first entry while the farthest is appended as the last. The frame data are fed into the hierarchical RNN with a single-layer perceptron as the feature extractor of each frame. The maximum number of frames is 150 for each sequence. A basketball possession is at most 24 seconds, so this results in an effective frame rate of 6.2fps. During training, we set both $\sigma_x$ and $\sigma_y$ to 2ft.

\textbf{SPORTLOGiQ NHL dataset.} This dataset includes the raw broadcast videos, player bounding boxes and trajectories with similar annotations to the NBA dataset. However, unlike the NBA dataset, the number of players in each frame may change due to the nature of broadcast videos. To solve this problem, we set a fixed number $N_p$ of players to use. If there are fewer than $N_p$ players, we zero out the extra entries. If there are more than $N_p$ players, we select the $N_p$ players that are most clustered. We essentially assume the players cluster around where the actions are. We use closeness centrality to implement this intuition. We build a complete graph over the players in a frame, each player being a node in the graph. Then we compute the closeness centrality for each node using Euclidean distance and choose the top $N_p$ highest closeness scores. 

Given the pixels inside the bounding box and the coordinates of a single player, we feed them into a VGG-16 network \cite{simonyan2014very} and a single-layer perceptron respectively. The outputs are then summed. This is repeated $N_p$ times (i.e.\ for every selected player), and finally we do element-wise max-pooling over the $N_p$ feature vectors to obtain a holistic feature representation for the $N_p$ players. Fig. \ref{fig:shared_net} outlines this workflow.

In the experiments, we use $N_p=4$. For each sequence, we use at most 80 frames for training and 200 frames for evaluation. After down-sampling the videos, the frame rate is 7.5fps. Thus the longest sequence allowed is approximately 10.7s for training and 26.7s for evaluation. We again use $\sigma_x=\sigma_y=2$ft.

\begin{figure}[t]
    \centering
    \includegraphics[width=0.42\textwidth]{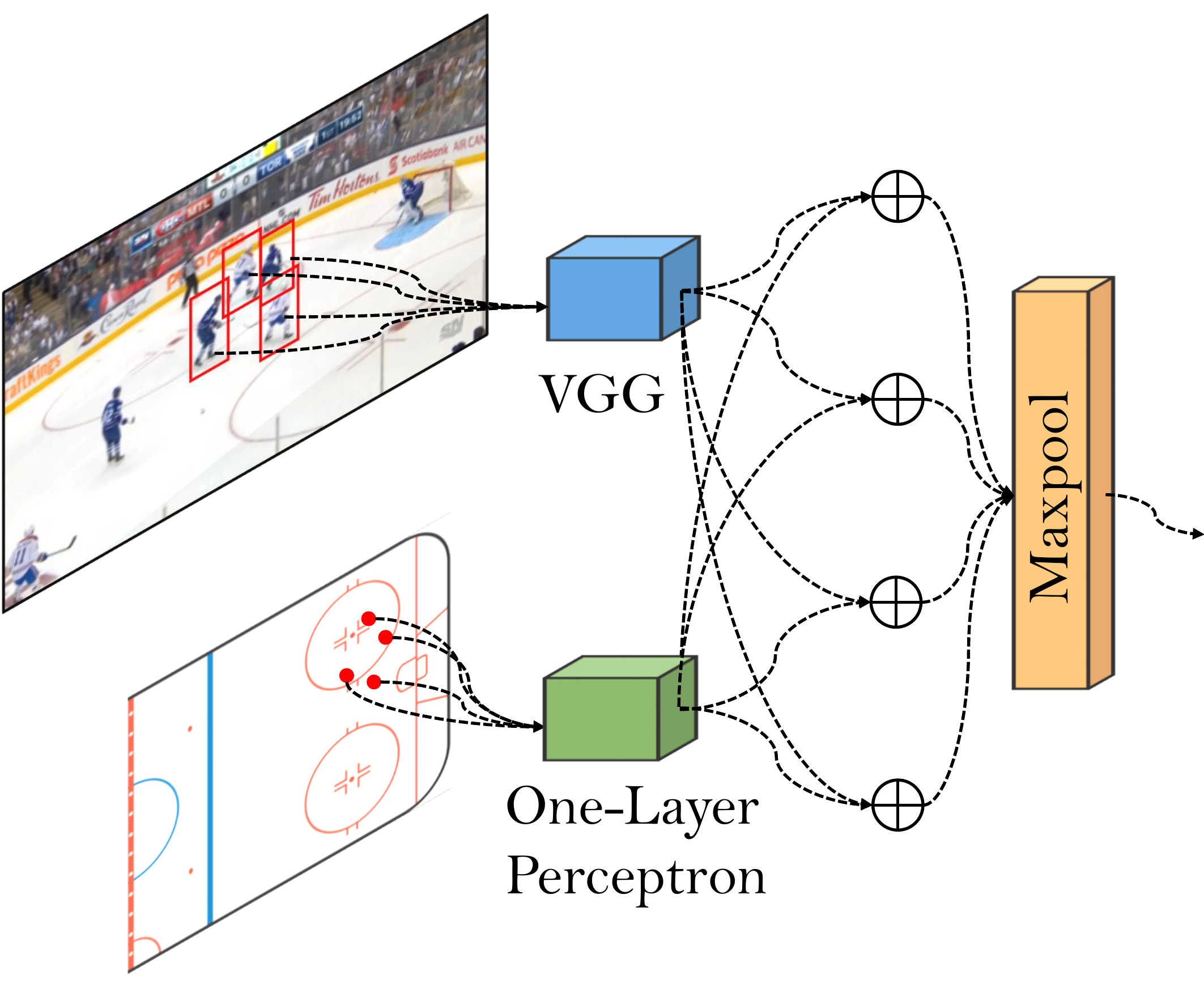}
    \caption{Frame-level feature extractor for the SPORTLOGiQ dataset.}
    \label{fig:shared_net}
\end{figure}

\subsection{Performance Measures}
We use mean absolute error (mAE) to evaluate the estimation of time and space, and mean average precision (mAP) to measure the performance of action category prediction. However, given the nature of sports games, there are significant variations among the time intervals between neighboring activities (intervals range from milliseconds to seconds). Reporting mAE alone ignores these variations. For example, an error of 100ms is considered less significant if the ground-truth time interval is 1s as opposed to merely 100ms. Therefore we advocate mean deviation rate (mDR) as a better measure. Deviation rate (DR) is calculated as below; mDR is DR averaged over all time steps. 
\begin{equation}
\label{eq:mdr}
    \textit{DR}=\frac{|\textit{Predicted Time}-\textit{Current Time}|}{\textit{Ground-Truth Time Interval}}
\end{equation}

\subsection{Baselines}
The baseline models predict the time of the next activity with conventional temporal point process, such as Poisson process, Hawkes process and self-correcting process. In order to predict the category and location of next activity, we utilize $k$-order Markov chains, where $k=1, 3, 5, 7, 9$. We do not use higher orders since most sample possessions do not have sequence length larger than 10. 

The inference stage of a $k$-order Markov chain works as follows. Given the most recent $k$ activities, we find the next activity with the highest transition probability. If the number of historical activities at current time step is less than $k$ or we are unable to find the exact $k$ historical activities in the transition matrix, we relax the dependency requirement by using the most recent $k-1$ activities. This is repeated until we find a valid transition to the next activity. The worst case is a degenerate Markov chain of 0-order, which is basically doing majority voting. Given the selected transition to next activity, we compute the mean space shift of all such transitions collected during training, which will be added to the current location, eventually making the prediction of the next activity location.

\subsection{Results}
The results in Tab.\ \ref{table:joint_on_time} show that the proposed TPMs outperform traditional statistical approaches. On the other hand, by comparing the two TPM variants, we find that TPM$_B$ performs better than TPM$_A$. Thus, the proposed conditional intensity $\lambda_B^*(t)$ can be more generic and effective than $\lambda_A^*(t)$. 

\begin{table}[t]
  \caption{Results of time prediction as part of joint estimation.}
  \label{table:joint_on_time}
  \centering
  \begin{tabular}{ccccc}
    \toprule
    & \multicolumn{2}{c}{NBA} & \multicolumn{2}{c}{NHL} \\
    \cmidrule{2-3}
    \cmidrule{4-5}
     & mAE (ms) & mDR (\%) & mAE (ms) & mDR (\%) \\
    \midrule
     TPM$_A$ & 288.1 & 54.6 & 527.9 & 174.5 \\
     TPM$_B$ & 282.1 & 52.0 & 530.7 & 172.0 \\
     Poisson & 365.5 & 547.0 & 645.4 & 297.6 \\
     Hawkes & 363.8 & 541.2  & 643.7 & 296.5 \\
     Self-Correcting & 382.4 & 522.4  & 643.0 & 291.5\\
     %Chance & 379.2 & 584.7\% & 656.2 & 315.8\% \\
    \bottomrule
  \end{tabular}
\end{table}

To see what the model has learned, we visualize the TPM$_B$ model predictions versus ground-truth annotations in Fig.\ \ref{fig:arrival_pattern}. We find that our model generally is able to approximate and keep track of the true arrival pattern in the input sequence (e.g., the upper row in each of the four subfigures in Fig.\ \ref{fig:arrival_pattern}). There are some large gaps between prediction and ground-truth when there comes a sudden high spike in the ground-truth. We believe this is because of the inherent randomness in sports games. In addition to the past series of activities, the action to be performed depends on many other factors such as tactics, which have not been explicitly observed and annotated during training and are challenging for the model to learn. 

The lower row of each of the four subfigures in Fig. \ref{fig:arrival_pattern} visualizes how the predicted time distribution changes as a basketball possession proceeds. The ability to capture the temporal distribution is a key advantage of the TPM.

\begin{figure*}[htp]
\centering
\tabcolsep 1pt
\begin{tabular}{cc}
\includegraphics[width=0.9\columnwidth]{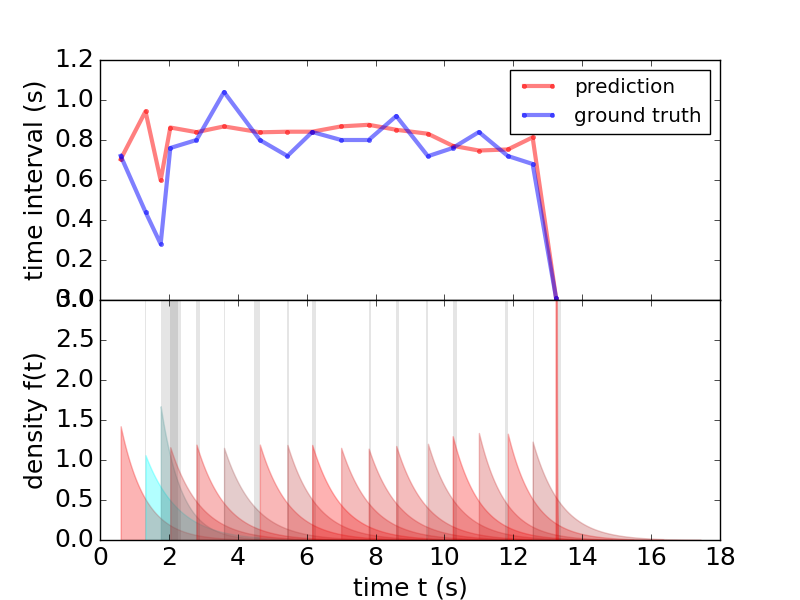} &
\includegraphics[width=0.9\columnwidth]{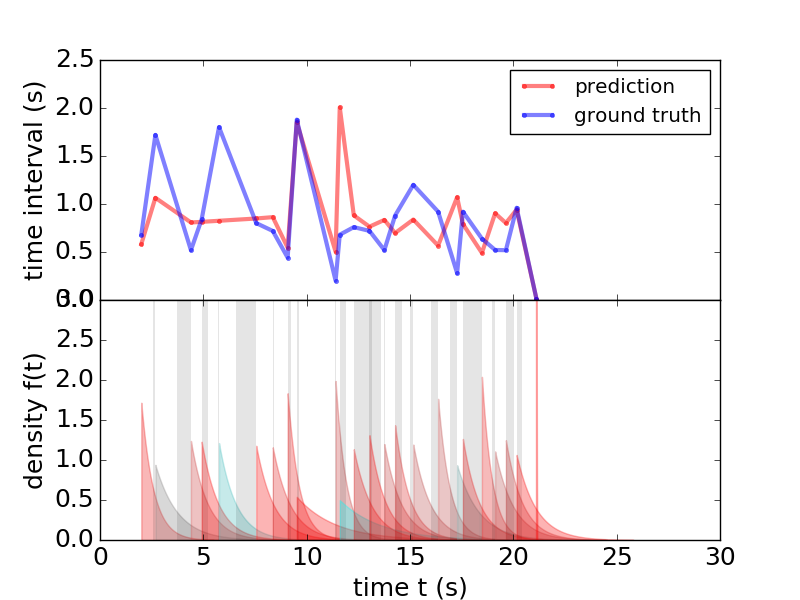} \\
\includegraphics[width=0.9\columnwidth]{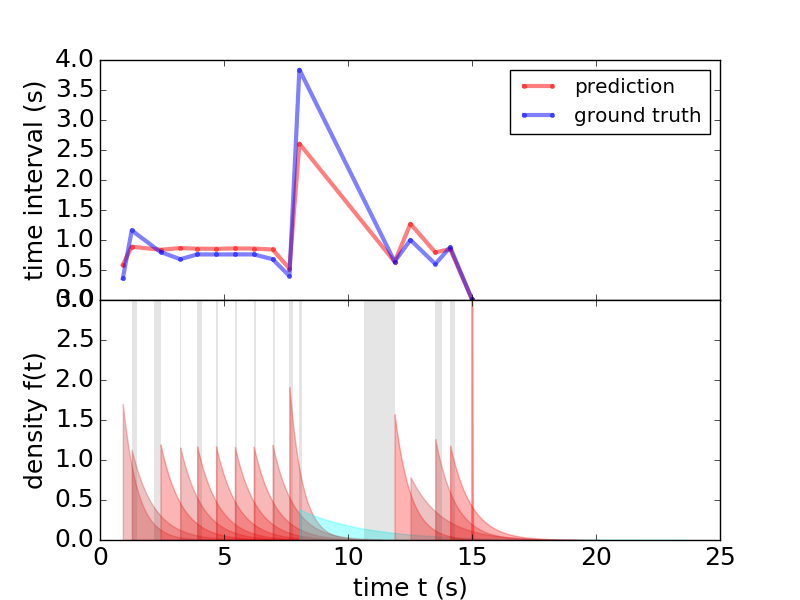} &
\includegraphics[width=0.9\columnwidth]{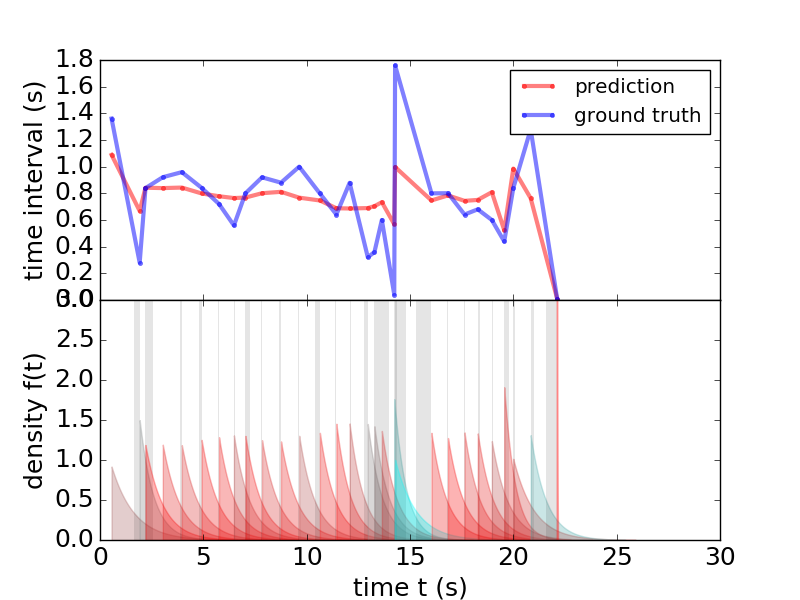} \\
\end{tabular}
\caption{Visualization of sample arrival patterns and predicted time distributions on the NBA dataset (TPM$_B$). The horizontal axis is the time line for a sequence of activity events within a basketball possession. The upper part of each subfigure plots the predicted and ground-truth time intervals between the current activity and the next activity. The lower part of each subfigure shows the predicted time distribution at each activity event (i.e. red or blue area). There is also a gray bar indicating the error between the predicted time and the ground-truth time on the next activity. The wider the gray bar, the more error and blueish the corresponding distribution; the thinner the gray bar, the less error and reddish the corresponding distribution. The near-vertical spiky distribution at the end of each subfigure shows how well TPM is predicting the sequence end.}
\label{fig:arrival_pattern}
\end{figure*}

\begin{figure}[htp]
  \centering
  \includegraphics[width=0.9\linewidth]{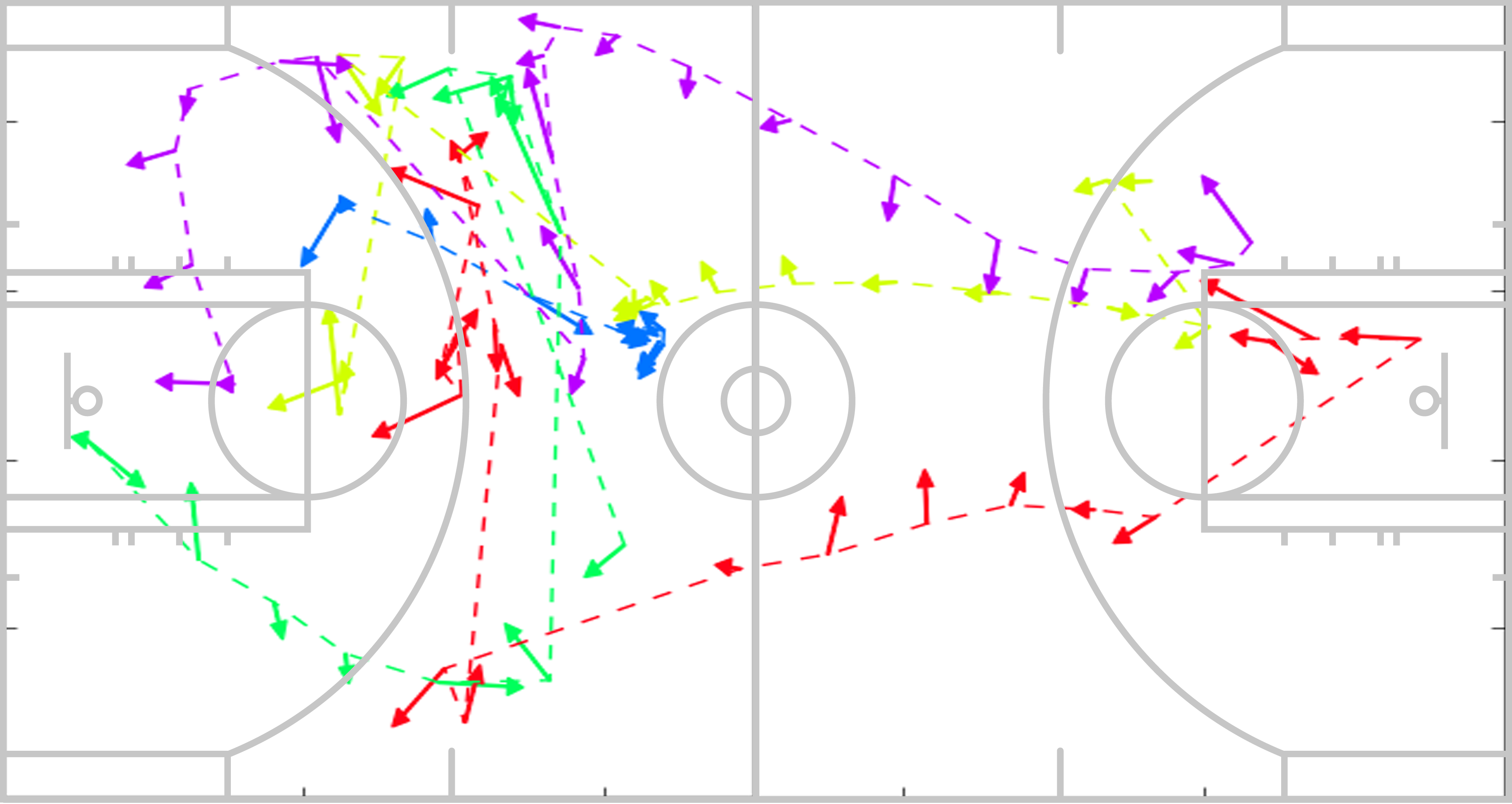}
  \caption{Qualitative results of space prediction on the NBA dataset. Multiple example possessions are shown, each in a different color. Ground-truth locations of the activity sequences are connected with dashed lines. Each arrow points from the ground-truth location of an activity to its location predicted by our model.}
  \label{fig:space}
\end{figure}

\begin{table*}[h]
    \caption{Results of category and space prediction as part of joint estimation. MC-$k$ refers to $k$-order Markov chain.}
    \label{table:joint_on_space_category}
    \centering
    \begin{tabular}{c|cc|ccccccc}
        \toprule
          \multicolumn{3}{c}{} & TPM$_A$ & TPM$_B$ & MC-1 & MC-3 & MC-5 & MC-7 & MC-9\\
         \midrule
          \multirow{8}{*}{NBA} & \multicolumn{2}{c|}{space mAE (ft)} & 3.43 & 3.28 & 6.91 & 6.86 & 6.73 & 6.69 & 6.69 \\
          \cmidrule{2-10}
          & \multirow{7}{*}{\rotatebox[origin=c]{90}{category AP (\%)}} & shoot & 57.9 & 58.0 & 10.1 & 32.9 & 35.7 & 37.0 & 37.4 \\
          & & dribble & 92.4 & 92.7 & 86.2 & 76.2 & 80.6 & 82.1 & 82.6 \\
          & & pass & 44.5 & 45.9 & 34.3 & 21.4 & 22.6 & 24.5 & 24.7 \\
          & & reception & 98.4 & 98.4 & 96.2 & 95.3 & 95.3 & 95.2 & 95.1 \\
          & & assist & 8.7 & 8.6 & 2.1 & 2.5 & 3.3 & 3.7 & 3.7 \\
          & & end & 99.9 & 99.9 & 99.9 & 99.9 & 99.9 & 99.9 & 99.9 \\
          %\cmidrule{3-10}
          & & mAP & 67.0 & 67.2 & 54.8 & 54.7 & 56.2 & 57.0 & 57.3 \\
          \midrule
          \multirow{12}{*}{NHL} & \multicolumn{2}{c|}{space mAE (ft)} & 56.95 & 57.01 & 65.96 & 66.60 & 66.85 & 66.88 & 67.24 \\
          \cmidrule{2-10}
          & \multirow{11}{*}{\rotatebox[origin=c]{90}{category AP (\%)}} & pass & 61.2 & 61.8 & 66.9 & 51.8 & 52.4 & 53.1 & 52.9 \\
          & & reception & 64.4 &64.3 & 78.8 & 50.8 & 51.8 & 52.3 & 52.1 \\
          & & carry & 21.3 & 21.2 & 30.8 & 20.0 & 18.7 & 19.2 & 18.8 \\
          & & shoot & 11.1 & 9.6 & 11.4 & 10.9 & 9.9 & 10.4 & 10.3 \\
          & & dumpin & 11.3 & 12.2 & 30.0 & 8.6 & 9.5 & 9.2 & 9.3 \\
          & & protection & 32.8 & 32.8 & 28.3 & 24.4 & 23.6 & 24.8 & 24.4 \\
          & & dumpout & 4.7 & 5.5 & 22.8 & 4.6 & 4.6 & 4.6 & 4.6 \\
          & & check & 11.0 & 11.8 & 19.5 & 7.2 & 7.3 & 8.0 & 8.7 \\
          & & block & 25.9 & 23.0 & 21.7 & 15.5 & 16.2 & 15.8 & 15.8 \\
          & & end & 80.6 & 79.5 & 47.0 & 32.9 & 26.6 & 25.0 & 25.0 \\
          %\cmidrule{3-10}
          & & mAP & 32.4 & 32.2 & 35.7 & 22.7 & 22.1 & 22.2 & 22.2 \\
         \bottomrule
    \end{tabular}
\end{table*}

In terms of space prediction, Tab. \ref{table:joint_on_space_category} shows quantitative results. We see that TPMs have consistently better performance than Markov chains on both datasets. A sample qualitative result is presented in Fig.\ \ref{fig:space}. Note that the court in NBA games is 94ft by 50ft and the rink in NHL games is 200ft by 85ft.

The space mAE (in Euclidean distance) on the NHL dataset is significantly greater than that on the NBA dataset. We believe this is because, in ice hockey games, players and the puck exhibit extremely quick motions. For example, the puck can be moved from one end of the rink to the other in less than a second, after which a puck reception could happen immediately, making the spatial location hard to predict. In contrast to hockey, our models are more accurate for basketball, where the relatively slower motions make space prediction more precise. Space prediction relies heavily on the speed of motion, but category prediction is not subject to such a constraint, so our models exhibit reasonable performance on inferring the type of the next activity. %Fig. \ref{fig:pr_curve} gives the precision-recall curve for category prediction, which is generated using TPM$_B$.

An interesting finding is that a $1^{st}$-order Markov chain has surprisingly good mAP on the NHL dataset when compared to Markov chains of other orders. After we look into the precision of each category (provided in the supplementary material), we find that it performs exceptionally better on activities such as carry, dumpout and dumpin, which are very rare in the training data as opposed to other types of activities. We did not observe similar behaviour on the NBA dataset, so we believe this results from the highly unbalanced ground-truth annotations in the NHL dataset. 

\begin{table}[t]
    \centering
    \caption{Comparison between TPM and a vanilla regression neural network on the task of predicting the time of next activity. Errors are measured in mDR. }
    \label{table:comp_regress}
    \begin{tabular}{ccc}
        \toprule
         & TPM$_B$ & Regression NN \\
        \midrule
         NBA & 51.6\% & 56.9\% \\
         NHL & 138.0\% & 188.2\% \\
        \bottomrule
    \end{tabular}
\end{table}

\section{Discussion}
\noindent \textbf{Regression v.s. distribution}. An intuitive way to predict the next activity time is training a regression neural network with mean squared error loss. However, we believe that learning a distribution captures more than regressing a scalar does. We validate this by doing a simple experiment. We train TPM$_B$ solely for time prediction. Everything else equal, we train a vanilla regression neural network to predict the time interval between current activity and next activity, which is then added to current timestamp to obtain the predicted time of next activity. Results are presented in Tab. \ref{table:comp_regress}. We see clearly how TPM does a better job in predicting the next activity occurrence. Additionally, since TPM is trained explicitly by maximizing the raw likelihood function, it readily enables us to inspect the temporal distribution of predictions as in Fig.\ \ref{fig:arrival_pattern}, whereas this feature is not available for a regression model.

\noindent \textbf{Framework and generality}. The proposed TPM is a general framework for prediction and modeling the arrival pattern of an activity sequence. It does not rely on a specific neural network structure. For example, in our experiment, we use a simple VGG-16 as the backbone network, but one can use other more advanced networks such as \cite{szegedy2015going, he2016deep, huang2017densely}. Networks \cite{msibrahiCVPR16deepactivity, tran2015learning, simonyan2014two, qiu2017learning} exclusively designed for action recognition can be used as well.
    
\noindent \textbf{Applicable scenarios}. TPM is a powerful model of the arrival pattern of sparsely distributed activities and can forecast the exact next activity time of occurrence. Here ``sparsely distributed" does not imply any concepts regarding weak supervision/annotation. TPM conforms to a fully supervised learning paradigm. Existing work such as \cite{huang2016connectionist} uses sparsely annotated data as well, but it addresses a totally different task than TPM. Furthermore, TPM specializes in dealing with sequences where activity events can be approximated as mass points in time. Activities with long temporal span do not fit into the TPM framework. Therefore, TPM is positioned in contrast to existing benchmarks such as Breakfast \cite{Kuehne12} and MPII-Cooking \cite{rohrbach15ijcv}, but useful for the sports analytics, surveillance, and autonomous vehicle scenarios outlined above.

%Another interesting finding is that the same model has better performance of time mDR when only trained for time prediction as opposed to joint training. We believe spatial location likelihood might be the cause that degrades time prediction. Because of the aforementioned difficulty in space prediction, the models could be misguided by the space prediction task during optimization, having negative effects on time prediction.

\section{Conclusion}

We have presented a novel take on the problem of activity forecasting.  Predicting when and where discrete, important activity events will occur is the task we explore.  In contrast with previous activity forecasting methods, this emphasizes semantically meaningful action categories and is explicit about when and where they will next take place.  We construct a novel hierarchical RNN based temporal point process model for this task.  Empirical results on challenging sports action datasets demonstrate the efficacy of the proposed methods.

% or


\appendices
\section{Probability density and cumulative distribution of temporal point processes} \label{appdx:pdf_cdf}
This seciton presents an intuitive derivation of Eq.\ref{eq:density} and Eq.\ref{eq:prob}.

The cumulative distribution $F^*(t)$ is defined as the probability that there is (at least) an event to happen at time $t$ since the last event time $t_j$. The ``*'' is a reminder that a quantity depends on the past. Let $f^*(t)$ denote the probability density function and $N(t)$ the number of events till time $t$. Then we have
\begin{equation}
    P(N(t)-N(t_j) \geq 1)=F^*(t).
\end{equation}

This is equivalent to
\begin{equation} \label{eq:p_num_events}
\begin{aligned}
    P(N(t)-N(t_j)=0)=& 1-F^*(t)
    %=& e^{ln(1-F^*(t))}\\
    %=& e^{-\int_{t_j}^{t} \frac{1}{1-F^*(u)} dF(u)}\\
    %=& e^{-\int_{t_j}^{t} \frac{f^*(u)}{1-F^*(u)} du}.
\end{aligned}
\end{equation}

Because the temporal point process models we are dealing with belong to the general class of non-homogeneous Poisson processes whose conditional intensity $\lambda^*(t)$ is a function of time $t$, by definition the number of events in $(t_j, t]$ conforms to Poisson distribution parameterized by $\Lambda$:
\begin{equation} \label{eq:p_poisson}
    P(N(t)-N(t_j)=k)=\frac{\Lambda^k}{k!} e^{-\Lambda},
\end{equation}
where $\Lambda$ is expected number of events per interval.

Because the conditional intensity $\lambda^*(t)$ is the expected rate of event arrivals, we have $\Lambda=\int_{t_j}^{t} \lambda^*(u) du$. Let $k$ in Eq. \ref{eq:p_poisson} be zero, then Eq. \ref{eq:p_poisson} is equal to Eq. \ref{eq:p_num_events}. This yields
\begin{equation}
    F^*(t)=1-e^{-\int_{t_j}^{t} \lambda^*(u) du},
\end{equation}
and that 
\begin{equation}
    f^*(t)=\frac{dF^*(t)}{dt}=\lambda^*(t) \cdot e^{-\int_{t_j}^{t} \lambda^*(u) du}.
\end{equation}

\section{The validity of conditional intensities} \label{appdx:valid_intensity}
This section provides the proof that the two conditional intensities (Eq.\ref{eq:exp_intensity} and Eq.\ref{eq:simple_intensity}) used in our experiments characterize valid temporal point processes. 
%For convenience, we use the notation $\lambda^*(t)=e^{vh_j+w(t-t_j)+b}, \ w \geq 0$ to denote Eq.\ref{eq:exp_intensity} and Eq.\ref{eq:simple_intensity}. When $w>0$, it becomes Eq.\ref{eq:exp_intensity}. When $w=0$, it becomes Eq.\ref{eq:simple_intensity}.
% \begin{prop}
% $\lambda^*(t)=e^{vh_j+w(t-t_j)+b}$ is a valid conditional intensity that defines a temporal point process if $w\geq0$.
% \end{prop}

\begin{proof}
$\lambda^*(t)$ takes the form of Eq.\ref{eq:exp_intensity} if $w>0$ while it takes the form of Eq.\ref{eq:simple_intensity} if $w=0$. Let us denote
\begin{equation}
    \Lambda^*(t)=\int_{t_j}^{t} \lambda^*(u) du=
    \begin{cases} 
        \frac{1}{w} e^{vh_j+b} \big[e^{w(t-t_j)}-1 \big], & w>0 \\
        e^{vh_j+b}(t-t_j), & w=0 \end{cases}
\end{equation}

When $w\geq0$, the quantity $\Lambda^*(t)$ is monotonically increasing in terms of $t$. As $t$ approaches infinity, $\Lambda^*(t)$ approaches infinity as well. Substituting $\Lambda^*(t)$ into Eq. \ref{eq:prob}, we have $F^*(\infty)=1-e^{-\Lambda^*(\infty)}=1$, so $\lambda^*(t)$ is a valid conditional intensity when $w \geq 0$. 

However, when $w<0$, we have $\Lambda^*(\infty)= -\frac{1}{w} \cdot e^{vh_j+b} $, hence $F^*(\infty)=1-e^{-\Lambda^*(\infty)}<1$. This definitely results in an invalid probability distribution. Therefore, $\lambda^*(t)=e^{vh_j+w(t-t_j)+b}$, or equivalently Eq.\ref{eq:exp_intensity} and Eq.\ref{eq:simple_intensity}, is valid if $w \geq 0$.
\end{proof}

\section{Inference of time}
In this section, we derive the predicted time $\hat{t}_{j+1}$ for the two conditional intensities (Eq.\ref{eq:exp_intensity} and Eq.\ref{eq:simple_intensity}) we used.

\subsection{When $\lambda^*(t)$ takes the form in Eq. \ref{eq:exp_intensity}}
%\begin{equation}
%\begin{aligned}
\begin{align*}
    \hat{t}_{j+1}=& E(t_{j+1})=\int_{t_j}^{\infty} t f^*(t) dt\\
    =& \int_{0}^{\infty} x e^{vh_j+wx+b + \frac{1}{w} e^{vh_j+b} (1-e^{wx})} dx + t_j \int_{t_j}^{\infty} f^*(t) dt\\
    &\text{\color{blue} (Obtained by letting $x=t-t_j$)}\\
    =& e^{vh_j+b} \int_{0}^{\infty} x e^{wx+ \frac{1}{w} e^{vh_j+b} (1-e^{wx})} dx + t_j \\
    &\text{\color{blue} ($\int_{t_j}^{\infty} f^*(t) dt$ is $F^*(\infty)$, so equal to 1)}\\
    =& \frac{e^{vh_j+b}}{w^2} \int_{1}^{\infty} \log y \cdot e^{\frac{1}{w} e^{vh_j+b} (1-y)} dy + t_j\\
    &\text{\color{blue} (Obtained by letting $y=e^{wx}$)}\\
    =& \frac{\eta e^\eta}{w} \int_{1}^{\infty} \log y \cdot e^{-\eta y} dy + t_j\\
    &\text{\color{blue} (where $\eta=\frac{e^{vh_j+b}}{w}$)}\\
    =& -\frac{e^\eta}{w} \int_{1}^{\infty} \log y \cdot de^{-\eta y} + t_j\\
    =& -\frac{e^\eta}{w} \big[ \log y \cdot e^{-\eta y} \Big|^{\infty}_{1} - \int_{1}^{\infty} \frac{e^{-\eta y}}{y} dy \big] + t_j\\
    &\text{\color{blue} (Integrate by parts)}\\
    =& \frac{\Gamma(0, \eta) \cdot e^{\eta}}{w}+t_j\\
    &\text{\color{blue} (where $\Gamma(0, \eta)=\int_{1}^{\infty} \frac{e^{-\eta y}}{y} dy$ is an incomplete} \\
    & \text{\color{blue} gamma function)}
\end{align*}
%\end{aligned}
%\end{equation}

\subsection{When $\lambda^*(t)$ takes the form in Eq. \ref{eq:simple_intensity}}
%\begin{equation}
%\begin{aligned}
\begin{align*}
    \hat{t}_{j+1}=& E(t_{j+1})=\int_{t_j}^{\infty} t f^*(t) dt\\
    =& \int_{t_j}^{\infty} t \lambda e^{-\lambda(t-t_j)} dt\\
    &\text{\color{blue} (Obtained by letting $\lambda=\lambda^*(t)$ since $\lambda^*(t)$ does not} \\
    &\text{\color{blue} actually rely on $t$)} \\
    =& \int_{0}^{\infty} x\lambda e^{-\lambda x} dx + t_j \int_{0}^{\infty} f^*(t)dt\\
    &\text{\color{blue} (Obtained by letting $x=t-t_j$)}\\
    =& \int_{0}^{\infty} x\lambda e^{-\lambda x} dx + t_j\\
    &\text{\color{blue} ($\int_{t_j}^{\infty} f^*(t) dt$ is $F^*(\infty)$, so equal to 1)}\\
    =& -\big[ e^{-\lambda}x \Big|_{0}^{\infty} -\int_{0}^{\infty} e^{-\lambda x}dx \big] +t_j\\
    &\text{\color{blue} (Integrate by parts)}\\
    =& \frac{1}{\lambda}+t_j\\
    =& e^{-(wh_j+b)} + t_j
\end{align*}

% % use section* for acknowledgment
% \ifCLASSOPTIONcompsoc
%   % The Computer Society usually uses the plural form
%   \section*{Acknowledgments}
% \else
%   % regular IEEE prefers the singular form
% \fi


\ifCLASSOPTIONcaptionsoff
  \newpage
\fi



\bibliographystyle{IEEEtran}
\bibliography{egbib}



% \end{thebibliography}



%\vfill


\end{document}